\documentclass[twoside,11pt]{article}
\usepackage{pgm} 

\usepackage{paralist}
\usepackage{amssymb,amsmath} 
\usepackage{stmaryrd}
\usepackage{graphicx}
\usepackage{color}
\usepackage{framed}
\usepackage[all,pdf]{xy}
\usepackage{amsmath,chemarrow}
\usepackage{hyperref,color,soul,booktabs}
\setulcolor{blue}
\newcommand{\BibTeX}{\textsc{B\kern-0.1emi\kern-0.017emb}\kern-0.15em\TeX}

\newcommand{\cg}[1]{{\cal{G}}^#1}

\newcommand{\given}{{ \; | \; }}

\newcommand{\V}{{\bf{V}}}

\newcommand{\sG}{\mathcal{G}}

\newcommand{\rsa}{{\rightsquigarrow}}
 \newcommand{\lsa}{{ \reflectbox{$\rightsquigarrow$}}}

\newcommand{\ra}{\rightarrow}

\newcommand{\lra}{\leftrightarrow}

\newcommand\independent{\protect\mathpalette{\protect\independenT}{\perp}} 
\def\independenT#1#2{\mathrel{\rlap{$#1#2$}\mkern2mu{#1#2}}}

\ShortHeadings{Causal Discovery from Subsampled Time Series Data by Constraint Optimization}{Hyttinen, Plis, J\"arvisalo, Eberhardt, Danks}

\begin{document}

\FrameSep0pt

\title{Causal Discovery from Subsampled Time Series Data\\by Constraint Optimization}

\author{\name Antti Hyttinen \email antti.hyttinen@helsinki.fi\\
\addr HIIT, Department of Computer Science, University of Helsinki
\AND 
\name Sergey Plis \email s.m.plis@gmail.com\\
\addr Mind Research Network and University of New Mexico
\AND 
\name Matti J\"arvisalo \email matti.jarvisalo@cs.helsinki.fi \\
\addr HIIT, Department of Computer Science, University of Helsinki
\AND 
\name Frederick Eberhardt \email fde@caltech.edu\\ 
\addr Humanities  and Social Sciences, California Institute of Technology
\AND 
\name David Danks \email ddanks@cmu.edu\\
\addr Department of Philosophy, Carnegie Mellon University
}

\editor{}

\maketitle

\begin{abstract}

This paper focuses on causal structure estimation from time series data in which measurements are obtained at a coarser timescale than the causal timescale of the underlying system. Previous work has shown that such subsampling can lead to significant errors about the system's causal structure if not properly taken into account. In this paper, we first consider the search for the system timescale causal structures that correspond to a given measurement timescale structure. We provide a constraint satisfaction procedure whose computational performance is several orders of magnitude better than previous approaches. We then consider finite-sample data as input, and propose the first constraint optimization approach for recovering the system timescale causal structure. This algorithm optimally recovers from possible conflicts due to statistical errors. More generally, these advances allow for a robust and non-parametric estimation of system timescale causal structures from subsampled time series data.

\end{abstract}

\begin{keywords}
causality; causal discovery; graphical models; time series; constraint satisfaction; constraint optimization.
\end{keywords}


\section{Introduction}\label{sec:intro}


Time-series data has long constituted the basis for causal modeling in many fields of science \citep{GrangerCausality,hamilton1994time, lutkepohl2005new}.
Despite the often very precise measurements 
at regular time points, the underlying causal interactions that give rise to the measurements often occur at a much faster timescale than the measurement frequency. While information about time order is generally seen as simplifying causal analysis, time series data that undersamples the generating process can be misleading about the true causal connections \citep{dash,iwasaki}. 
For example, Figure~\ref{fig:dynamic}a shows the causal structure of a 
process unrolled over discrete time steps, and Figure~\ref{fig:dynamic}c shows the corresponding structure of the same process, obtained by marginalizing every second time step. If the subsampling rate is not taken into account, we might conclude that optimal control of $V_2$ requires interventions on both $V_1$ and $V_3$, when the influence of $V_3$ on $V_2$ is, in fact, completely mediated by $V_1$ (and so intervening only on $V_1$ suffices).

Standard methods for estimating causal structure from time series either focus exclusively on estimating a transition model at the measurement timescale (e.g.,\ Granger causality \citep{GrangerCausality,granger1980testing}) or combine a model of measurement timescale transitions with so-called ``instantaneous'' or ``contemporaneous'' causal relations that (are supposed to) capture any interactions that are faster than the measurement process (e.g.,\ 
SVAR) \citep{lutkepohl2005new,hamilton1994time,hyvarinen2010estimation}. In contrast, we follow \cite{Plis2015a,PlisDY:UAI2015} and \cite{Gong2015},
and explore the possibility of identifying (features of) the causal process at the true timescale from data that subsample this process.

In this paper, we provide an exact inference algorithm based on using a general-purpose Boolean constraint solver~\citep{DBLP:series/faia/2009-185,DBLP:journals/aicom/GebserKKOSS11}, and demonstrate that it is orders of magnitudes faster than the current state-of-the-art method by \cite{PlisDY:UAI2015}. 
At the same time, our 
approach is much simpler and allows inference in more general settings.
We then show how the approach naturally integrates possibly conflicting results obtained from the data. Moreover, unlike the approach by \cite{Gong2015}, our method does not depend on a particular parameterization of the underlying model and scales to a more reasonable number of variables.

\section{Representation}\label{sec:representation}


We assume that the system of interest relates a set of variables $\V^t = \{V_1^t,\ldots,V_n^t\}$ defined at discrete time points $t \in \mathbb{Z}$ with continuous ($\in \mathbb{R}^n$) or discrete ($\in \mathbb{Z}^n$) values \citep{DorisEntner}. We distinguish the representation of the true causal process at the \emph{system timescale} from the time series data that are obtained at the \emph{measurement timescale}. Following \cite{PlisDY:UAI2015}, we assume that the true between-variable causal interactions at the system timescale constitute a first-order Markov process; that is, that the independence $\mathbf{V}^t \independent \mathbf{V}^{t-k} |\mathbf{V}^{t-1}$ holds for all  $k > 1$. The parametric models for these causal structures are structural vector autoregressive (SVAR) processes or dynamic (discrete/continuous variable) Bayes nets. 
Since the system timescale can be arbitrarily fast (and causal influences take time),  we assume that there is no ``contemporaneous'' causation of the form $V_i^t \ra V_j^t$~\citep{granger1988some}. We also assume that $\V^{t-1}$ contains all common causes of variables in $\V^{t}$. These assumptions jointly express the widely used causal sufficiency assumption (see \cite{sgs1993}) in the time series setting.

The system timescale causal structure can thus be represented by a causal graph $G^1$ 
consisting (as in a dynamic Bayes net) only of arrows of the form $V_i^{t-1} \rightarrow V_j^{t}$, where $i=j$ is permitted (see Figure~\ref{fig:dynamic}a for an example). Since the causal process is time invariant, the edges repeat through $t$.    
In accordance with 
\cite{PlisDY:UAI2015}, for any $G^1$ we use a simpler, rolled graph representation, denoted by $\cg{1}$, where $V_i \ra V_j \in \sG^1$ iff $V_i^{t-1} \ra V_j^{t} \in G^1$. Figure~\ref{fig:dynamic}b shows the rolled graph representation $\cg{1}$ of $G^1$ in Figure~\ref{fig:dynamic}a.



Time series data are obtained from the above process at the \emph{measurement timescale}, given by some (possibly unknown) integral sampling rate $u$. The measured time series sample $\V^t$ is at times $t, t-u, t-2u,\ldots$; we are interested in the case of $u > 1$, i.e., the case of subsampled data. A different route to subsampling would use continuous-time models as the underlying system timescale structure. However, some series (e.g., transactions such as salary payments) are inherently discrete time processes \citep{Gong2015}, and many continuous-time systems can be approximated arbitrarily closely as discrete-time processes. Thus, we focus here on discrete-time causal structures as a justifiable, and yet simple, basis for our non-parametric inference procedure.

The structure of this subsampled time series can be obtained from $G^1$ by marginalizing the intermediate time steps. Figure~\ref{fig:dynamic}c shows the measurement timescale structure $G^2$ corresponding to subsampling rate $u=2$ for the system timescale causal structure in Figure~\ref{fig:dynamic}a. Each directed edge in $G^2$ corresponds to a directed path of length 2 in $G_1$. For arbitrary $u$, the formal relationship between $G^u$ and $G^1$ edges is

\begin{center}
$V_i^{t-u} \ra V_j^t \in G^u \ \Leftrightarrow \  V_i^{t-u} \rsa V_j^{t} \in G^1$,
where $\rsa$ denotes a directed path.\footnote{We assume a type of faithfulness assumption (see~\cite{sgs1993}), such that influences along (multiple) paths between nodes do not exactly cancel in $G^u$.}
\end{center}

\begin{figure}[!t]
\centering
\vspace{-10mm}
\hspace*{1mm}\resizebox{0.65\columnwidth}{!}{$
\begin{array}{cccc}\xymatrix@R=10pt@C=20pt{
\cdots & 
*++[F-:<10pt>]{V_1^{t-2}}  \ar[rdd] \ar[r] 
&
 *++[F-:<10pt>]{V_1^{t-1}}  \ar[rdd] \ar[r] 
& 
*++[F-:<10pt>]{\;\;V_1^{t}\;\;}
& \cdots
\\
\\
\cdots&
*++[F-:<10pt>]{V_2^{t-2}}  \ar[rdd] 
&
 *++[F-:<10pt>]{V_2^{t-1}}  \ar[rdd] 
 &
*++[F-:<10pt>]{\;\;V_2^{t}\;\;}   
& \cdots 
\\
\\
\cdots&
*++[F-:<10pt>]{V_3^{t-2}} \ar[ruuuu]  
&
 *++[F-:<10pt>]{V_3^{t-1}}  \ar[ruuuu] 
 &
*++[F-:<10pt>]{\;\;V_3^{t}\;\;}   
& \cdots
\\
}
& \xymatrix@R=10pt@C=20pt{ \\
& 
&
 *++[F-:<10pt>]{V_1} \ar@(ul,ur)[] \ar[rdddd] 
& 
& \\
\\
&&&&\\
\\
& *++[F-:<10pt>]{V_3}  \ar[ruuuu]   
&
 &
*++[F-:<10pt>]{V_2} \ar[ll]  
& \\
} \\ \\
\text{a) Unrolled graph } G^1 &
\text{b) Rolled graph } \cg{1} \\
\text{(system timescale)} &
\text{(system timescale)} \\ \\
 \xymatrix@R=10pt@C=20pt{
\cdots &
*++[F-:<10pt>]{V_1^{t-2}} \ar[rr]  \ar[rrdd] \ar[rrdddd] 
& \;\;\;\;\;\;\;\;\;\;
& 
*++[F-:<10pt>]{\;\;V_1^{t}\;\;}
& \cdots
\\
\\
\cdots& 
*++[F-:<10pt>]{V_2^{t-2}}  
\ar@[red]@{<->}[uu] 
\ar[rruu] 
&
 &
*++[F-:<10pt>]{\;\;V_2^{t}\;\;}  \ar@[red]@{<->}[uu] 
& \cdots 
\\
\\
\cdots& 
*++[F-:<10pt>]{V_3^{t-2}} \ar[rruu] \ar[rruuuu] 
&
 &
*++[F-:<10pt>]{\;\;V_3^{t}\;\;}   
& \cdots
\\
 } & \xymatrix@R=10pt@C=20pt{ \\
& 
&
 *++[F-:<10pt>]{V_1} \ar@(ul,ur)[] \ar@/^1pc/[rdddd]   \ar@/^1pc/[ldddd] 
& 
& \\
\\
&&&&\\
\\
& *++[F-:<10pt>]{V_3}\ar[rr]   \ar@/^1pc/[ruuuu]  
&
 &
*++[F-:<10pt>]{V_2}  \ar@[red]@{<->}@/_3pc/[luuuu]  \ar@/^1pc/[luuuu] 
& }
\\ \\
 \text{c) Unrolled graph } G^2 &
 \text{d) Rolled graph } \cg{2}\\
 \text{(measurement timescale)} &
\text{(measurement timescale)}
\end{array}$}
\caption{Example graphs: a)~$G^1$, b)~$\sG^1$, c)~$G^u$, d)~$\sG^u$ with 
subsampling rate $u=2$.\label{fig:dynamic}}
\end{figure}
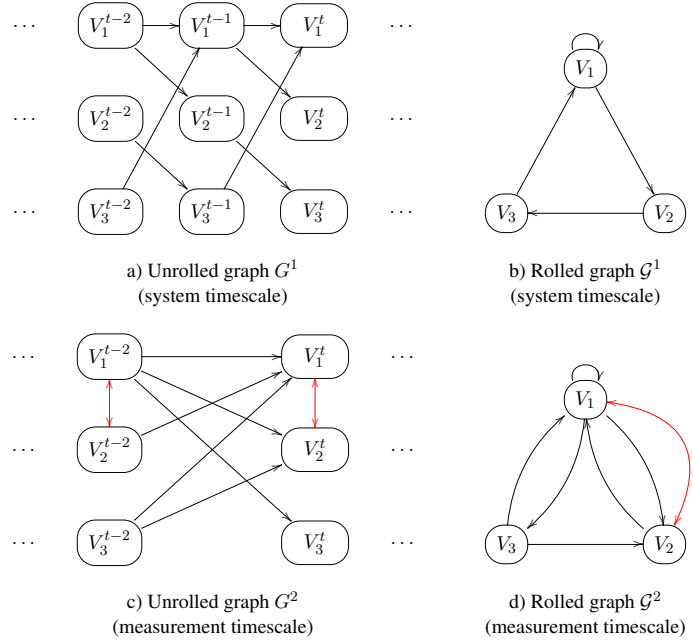

Subsampling a time series additionally induces ``direct'' dependencies between variables in the same time step \citep{wei1994time}. The bi-directed arrow $V_1^t \leftrightarrow V_2^t$ in Figure~\ref{fig:dynamic}c is an example: $V_1^{t-1}$ is an unobserved (in the data) common cause of $V_1^t$ and $V_2^t$ in $G^1$ (see Figure~\ref{fig:dynamic}a). Formally, the system timescale structure $G^{1}$ induces bi-directed edges in the measurement timescale $G^{u}$ for $i \neq j$ as follows:
\begin{eqnarray*}
V^t_i \lra V^t_j \in G^u &\Leftrightarrow& \exists(V^t_i \lsa V_c^{t-k} \rsa V_j^t) \in G^1,k<u.
\end{eqnarray*}
Just  as $\cg{1}$  represents the  rolled version  of $G^1$,  $\cg{u}$
represents the rolled  version of $G^u$: $V_i \ra V_j  \in \cg{u}$ iff
$V_i^{t-u} \ra  V_j^{t} \in  G^u$ and  $V_i \lra  V_j \in  \cg{u}$ iff
$V_i^{t} \lra V_j^{t} \in G^u$.

The relationship  between $\cg{1}$ and $\cg{u}$---that  is, the impact
of subsampling---can be concisely represented using only the rolled
graphs:
\begin{eqnarray}
V_i \ra V_j \in \cg{u} &\Leftrightarrow & V_i \overset{u}{\rsa}
V_j \in \cg{1} \label{eq1}\\
V_i \lra V_j \in \sG^u  &\Leftrightarrow&   \exists (V_i \overset{<u}{\lsa} V_c \overset{<u}{\rsa} V_j) \in \cg{1}, i\neq j \label{eq2}
\end{eqnarray}
where $\overset{u}{\rsa}$ denotes a path of length $u$ and $\overset{<u}{\rsa}$ denotes a path shorter than $u$ (of the same length on each arm of a common cause). Using the rolled graph notation, the logical encodings in Section~\ref{sec:gu2g1} are considerably simpler.




\label{sec:probdef}
\cite{Danks2013} demonstrated that, in the infinite sample limit, the causal structure $\cg{1}$ at the system timescale is in general underdetermined, even when the subsampling rate $u$ is known and small. Consequently, even when ignoring estimation errors, the most we can learn is an equivalence class of causal structures at the system timescale. We define $\cal{H}$ to be the estimated version of $\cg{u}$, a graph over $\V$ obtained or estimated at the measurement timescale (with possibly unknown $u$). Multiple $\cg{1}$ can have the same structure as $\cal{H}$ for distinct $u$, which poses a particular challenge when $u$ is unknown. If $\cal{H}$ is estimated from data, it is possible, due to statistical errors, that no $\cg{u}$ has the same structure as $\cal{H}$. With these observations, we are ready to define the computational problems focused on in this work.
\medskip

\noindent
\textbf{Task 1} \emph{Given a measurement timescale structure $\cal{H}$ (with possibly unknown $u$), infer the (equivalence class of) causal structures $\cg{1}$ consistent  with $\cal{H}$ (i.e. $\cg{u}=\cal{H}$ by Eqs.~\ref{eq1} and~\ref{eq2}).}
\medskip


\noindent
We also consider the corresponding problem when the subsampled time series is directly provided as input, rather than $\cg{u}$.
\medskip

\noindent
\textbf{Task 2} \emph{Given a dataset of measurements of $\V$ obtained at the measurement timescale (with possibly unknown $u$), infer the (equivalence class of) causal structures $\cg{1}$ (at the system timescale) that are (optimally) consistent with the data.}
\medskip

\noindent
Section~\ref{sec:gu2g1} provides a solution to Task~1, and Section~\ref{sec:data2g1} provides a solution to Task~2.

\section{Finding Consistent $\cg{1}$s} \label{sec:gu2g1}

We first focus on Task~1.  We discuss the computational complexity of the underlying decision problem,
and present a practical Boolean constraint satisfaction approach
that  empirically scales up to
 significantly larger graphs than previous
state-of-the-art algorithms.

\subsection{On Computational Complexity}

Considering the task of finding a single $\cg{1}$ 
 consistent with a given $\cal{H}$,  a variant of the associated decision problem
 is related to the NP-complete problem of finding a matrix root.

  \setlength{\topsep}{5pt}
\begin{theorem}
\label{thm:npc}
Deciding whether there is a $\cg{1}$ that is consistent with the  directed edges of a given $\cal{H}$ is NP-complete for any fixed $u\ge 2$.
\end{theorem}

\begin{proof}
Membership in NP follows from a guess and check:
guess a candidate $\cg{1}$, and deterministically check
whether the length-$u$ paths of $\cg{1}$ correspond to the edges of $\cal{H}$~\citep{PlisDY:UAI2015}.
For NP-hardness, for any fixed $u\ge 2$, there is a straightforward reduction from 
the NP-complete problem of determining whether a Boolean $B$ matrix has a $u$th root~\citep{kutz2004complexity}\footnote{Multiplication of two values in $\{0,1\}$ is defined as the logical-or, or equivalently, the maximum operator.} 
for a given $n \times n$ Boolean matrix $B$, interpret $B$ as the directed edge relation of
$\cal{H}$, i.e., $\cal{H}$ has the edge $(i,j)$ iff $A^u(i,j)=1$. It is then easy to see
that there is a $\cg{1}$ that is consistent with the obtained $\cal{H}$ iff 
$B=A^u$ for some binary matrix $A$ (i.e., a $u$th root of $B$).
\end{proof}

If $u$ is unknown, then membership in NP can be established in the same way 
by guessing both a candidate $\cg{1}$ and a value for $u$. 
Theorem~1 
ignores the possible bi-directed edges in $\cal{H}$ (whose presence/absence is also harder to determine reliably from practical sample sizes; see Section~\ref{sec:simulations}).
Knowledge of the presences and absences of 
such edges in $\cal{H}$ can 
restrict the set of candidate $\cg{1}$s. 
For example, in the special case where $\cal{H}$ is known to not contain \emph{any}
bi-directed edges, the possible $\cg{1}$s have a fairly simple structure:
in any $\cg{1}$ that is consistent with $\cal{H}$,
every node has at most one successor.\footnote{To see this,
assume $X$ has two successors, $Y$ and $Z$, s.t. $Y \neq Z$ in $\cg{1}$.  Then $\cg{u}$  will contain a bi-directed edge $Y \leftrightarrow Z$ for all $u \geq 2$, which contradicts the assumption that $\cal{H}$ has no bi-directed edges.}
Whether this knowledge can be
used to prove a more fine-grained complexity result for special cases is an open question.

\subsection{A SAT-Based Approach}

Recently, the first exact search algorithm for finding 
the $\cg{1}$s that are consistent with a given $\cal{H}$ for a known $u$ was presented by~\cite{PlisDY:UAI2015}; it represents the current state-of-the-art. Their approach implements
a specialized depth-first search procedure for the problem, with domain-specific polynomial 
time search-space pruning techniques.
As an alternative, we present here
a Boolean satisfiability based approach. 
First, we represent the problem exactly using 
a rule-based constraint
satisfaction formalism.
Then, for
a given input $\cal{H}$, we employ an off-the-shelf Boolean constraint satisfaction solver
for finding a $\cg{1}$ that is guaranteed to be consistent with $\cal{H}$ (if such $\cg{1}$ exists).
Our approach is not only simpler than the approach of~\cite{PlisDY:UAI2015},
but as we will show, it also significantly improves the
current state-of-the-art in runtime efficiency and scalability.

 We use here answer set programming (ASP) as the constraint satisfaction formalism \citep{DBLP:journals/amai/Niemela99, DBLP:journals/ai/SimonsNS02, DBLP:journals/aicom/GebserKKOSS11}. 
It offers an expressive
declarative modelling language, in terms of first-order logical rules,
for  various types  of NP-hard  search and  optimization problems.  To
solve a problem via ASP, one first needs to develop an ASP program (in
terms of ASP rules/constraints) that  models the problem at hand; that
is, the declarative rules implicitly represent the set of solutions to
the problem in  a precise fashion.  Then one or  multiple (optimal, in
case of optimization  problems) solutions to the  original problem can
be  obtained by  invoking an  off-the-shelf  ASP solver,  such as  the
state-of-the-art                                                \texttt{Clingo}
system~\citep{DBLP:journals/aicom/GebserKKOSS11}  used  in  this  work.
The  search  algorithms implemented  in the  \texttt{Clingo} system  are
extensions of state-of-the-art Boolean satisfiability and optimization
techniques
which
can today outperform even  specialized domain-specific algorithms, as we
show here.

We proceed by describing a simple ASP encoding of 
the problem of finding
a $\cg{1}$ that is consistent with a given $\cal{H}$. 
The input---the measurement timescale structure $\cal{H}$---is represented
as follows.
The input predicate
\texttt{\small node/1} represents the nodes of $\cal{H}$ (and all graphs),
indexed by $1\ldots n$.
The presence of a directed edge $X \rightarrow Y$ between nodes $X$ and $Y$
is represented using the predicate  \texttt{\small edgeh/2}
as \texttt{\small edgeh(X,Y)}. Similarly, the fact that  
an edge $X \rightarrow Y$  is not present is represented using the 
predicate 
\texttt{\small no\_edgeh/2} as 
\texttt{\small no\_edgeh(X,Y)}.
The presence of a bidirected edge $X \leftrightarrow Y$ between nodes $X$ and $Y$
is represented using the predicate  \texttt{\small confh/2}
as  \texttt{\small confh(X,Y)} ($X<Y$), and the fact that 
an edge $X \leftrightarrow Y$  is not present is represented using the
predicate
\texttt{\small no\_confh/2} as
\texttt{\small no\_confh(X,Y)}.

If $u$ is known, then it can be passed as input using \texttt{\small u(U)}; 
alternatively, it can be defined as a single value in a given range (here set to $1,\ldots,5$ as an example):

\begin{framed}
\footnotesize  
\begin{verbatim}                                                                                                                                              
  urange(1..5). % Define a range of u:s                                                                                                                                                    
                                                                                                                     
  1 { u(U): urange(U) } 1. % u(U) is true for only one U in the range                                                                                                                                      
\end{verbatim}
\end{framed}

Solution $\cg{1}$s  are represented via the predicate \texttt{\small edge1/2},
where \texttt{\small edge1(X,Y)} is \emph{true} iff $\cg{1}$
contains the edge $X\rightarrow Y$.
In ASP, the set of candidate solutions (i.e., the 
set of all directed graphs over $n$ nodes) over which the search for solutions is
performed, is declared via the so-called \emph{choice construct} within the following rule,
stating that candidate solutions may contain directed edges between any pair of nodes.

\begin{framed}
\footnotesize  
\begin{verbatim} { edge1(X,Y) } :- node(X), node(Y).
\end{verbatim}
\end{framed}                                                                                                                                                  
                                                                                                                                                              
The measurement timescale structure $\cg{u}$ corresponding to the candidate solution $\cg{1}$
 is represented using the predicates \texttt{\small edgeu(X,Y)} and \texttt{\small confu(X,Y)}, which are derived in the following way.                                           
First,                               
we declare the mapping from a given $\cg{1}$ to the corresponding $\cg{u}$ by 
declaring the exact length-$L$ paths in a non-deterministically chosen
candidate solution $\cg{1}$.
For this, we declare rules that compute  the length-$L$ paths
inductively for all $L\le U$,
 using the 
 predicate \texttt{\small path(X,Y,L)} to represent
that there is a length-$L$ path from $X$ to $Y$.

\begin{framed}                                                                                                                                                
\footnotesize                                                                                                                                                          
\begin{verbatim}  
  % Derive all directed paths up to length U   
  path(X,Y,1) :- edge1(X,Y).   
  path(X,Y,L) :- path(X,Z,L-1), edge1(Z,Y), L <= U, u(U).
\end{verbatim}                                                                                                                                                
\end{framed}                                             

Second, to obtain $\cg{u}$, 
we encode Equations~\ref{eq1} and~\ref{eq2}
with the following rules
that form predicates \texttt{\small edgeu/2} and \texttt{\small confu/2} describing the edges $\cg{1}$
induces on the measurement timescale structure.
               
\begin{framed}                                                                                                                                                
\footnotesize                                                                                                                                                      
\begin{verbatim}  % Paths of length U, correspond to measurement timescale edges                                                                                                                               
  edgeu(X,Y) :- path(X,Y,L), u(L).

  % Paths of equal length (<U) from a single node result in bi-directed edges
  confu(X,Y) :- path(Z,X,L), path(Z,Y,L), node(X;Y;Z), X < Y, L < U, u(U).
\end{verbatim}                                                                                                                                                
\end{framed}                                                                                                                                                  

Finally, we declare constraints that require that
the $\cg{u}$ represented by the
 \texttt{\small edgeu/2} and \texttt{\small confu/2} predicates 
is consistent with the input $\cal{H}$.                
This is achieved with the following rules, which enforce that
the edge relations of $\cg{u}$ and $\cal{H}$
are exactly the same for any solution $\cg{1}$.

\begin{framed}                                                                                                                                                
\footnotesize                                                                                                                                                          
\begin{verbatim}                                                                                                                                              
  :- edgeh(X,Y), not edgeu(X,Y).
  :- no_edgeh(X,Y), edgeu(X,Y).
  :- confh(X,Y), not confu(X,Y).
  :- no_confh(X,Y), confu(X,Y).
\end{verbatim}                                                                                                                                                
\end{framed}   

\noindent
Our ASP encoding of Task~1 consists of the rules just described.
The set of solutions of the encoding correspond exactly to the 
$\cg{1}$s  consistent with the input $\cal{H}$. 

\begin{figure}[t]
\vspace{-5mm}
\begin{center}
 \includegraphics[width=0.4\columnwidth, bb= -1 -1 406 372]{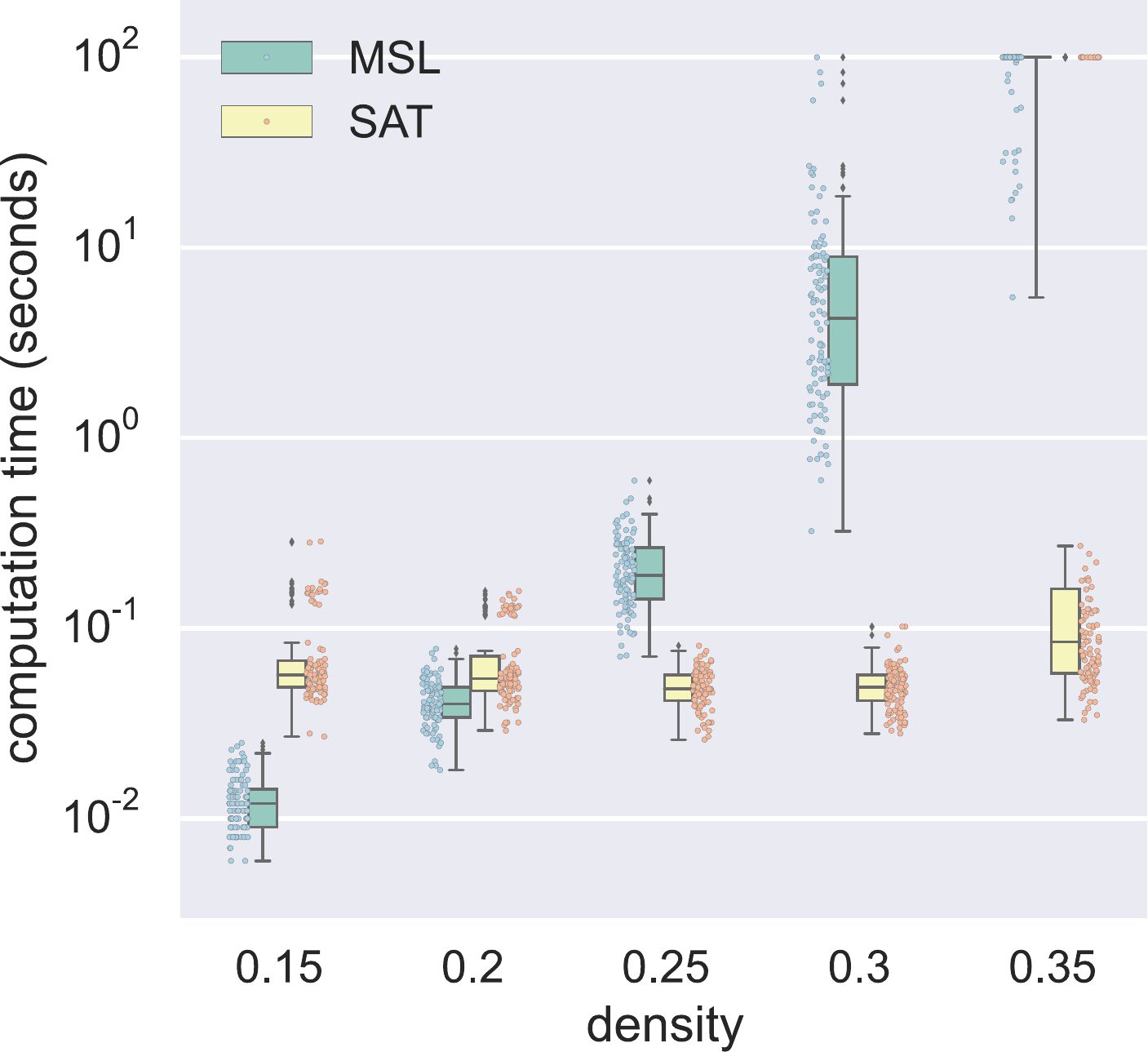}\quad\includegraphics[width=0.4\columnwidth, bb=-1 -1 406 383]{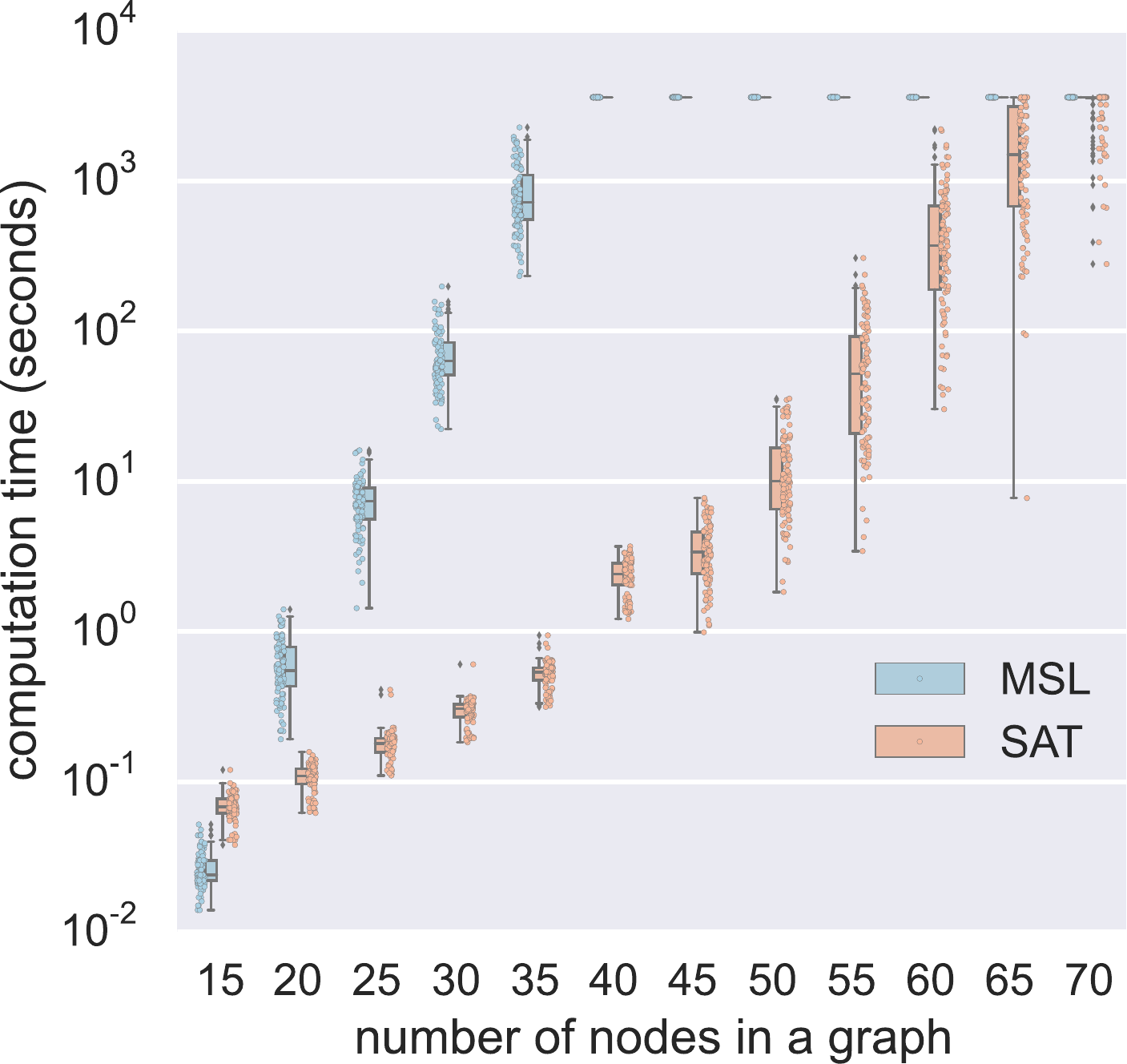}
\end{center}
\vspace{-15pt}
\caption{ Running times. Left: for 10-node  graphs as  a  function of graph density (100 graphs  per density and a timeout of
  100 seconds); Right: for 10\%-dense graphs as a
  function of  graph size (100 graphs  per density and a  timeout of 1
  hour). }
\label{fig:scalability-density}
\label{fig:scalability-nodes}
\vspace{-5mm}
\end{figure}

\subsection{Runtime Comparison}

Both our proposed SAT-based approach and the 
recent specialized search algorithm MSL~\citep{PlisDY:UAI2015}
are correct and complete, so we focus on differences in efficiency,  using the implementation of MSL by the original authors.
Our approach allows for searching simultaneously over a range of values of $u$,
but 
\cite{PlisDY:UAI2015} focused on the case $u=2$;
hence, we restrict the comparison to $u=2$.

We simulated system timescale graphs with varying density and number of nodes (see Section~\ref{sec:modelgeneration} for exact details), and
then generated the measurement timescale structures for subsampling rate $u=2$. This structure was given as input to the inference procedures. Note that the input consisted here of graphs for which there always is a $\cg{1}$, so all instances were satisfiable. The task of the algorithms was to output up to  1000 (system timescale) graphs in the equivalence 
class. The ASP encoding was solved by \texttt{Clingo} using the flag \texttt{-n 1000} for the solver to enumerate  1000 solution graphs (or all,  in cases where there were less than 1000 solutions).

\begin{figure}[b]
\vspace{-3mm}
\begin{center}
 \includegraphics[width=0.45\columnwidth, bb= 0 0 426 286]{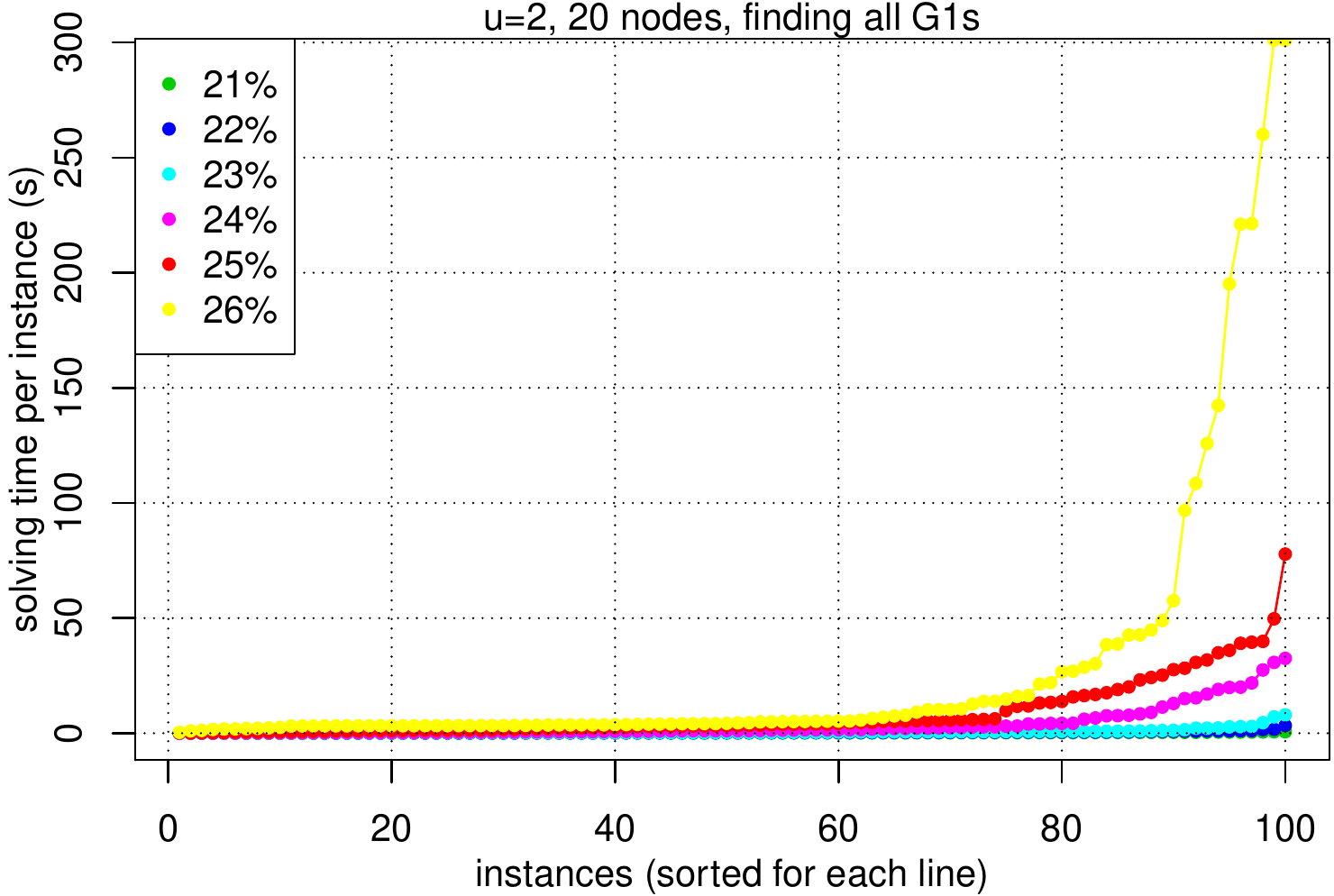}\quad\includegraphics[width=0.45\columnwidth,bb= 0 0 426 286]{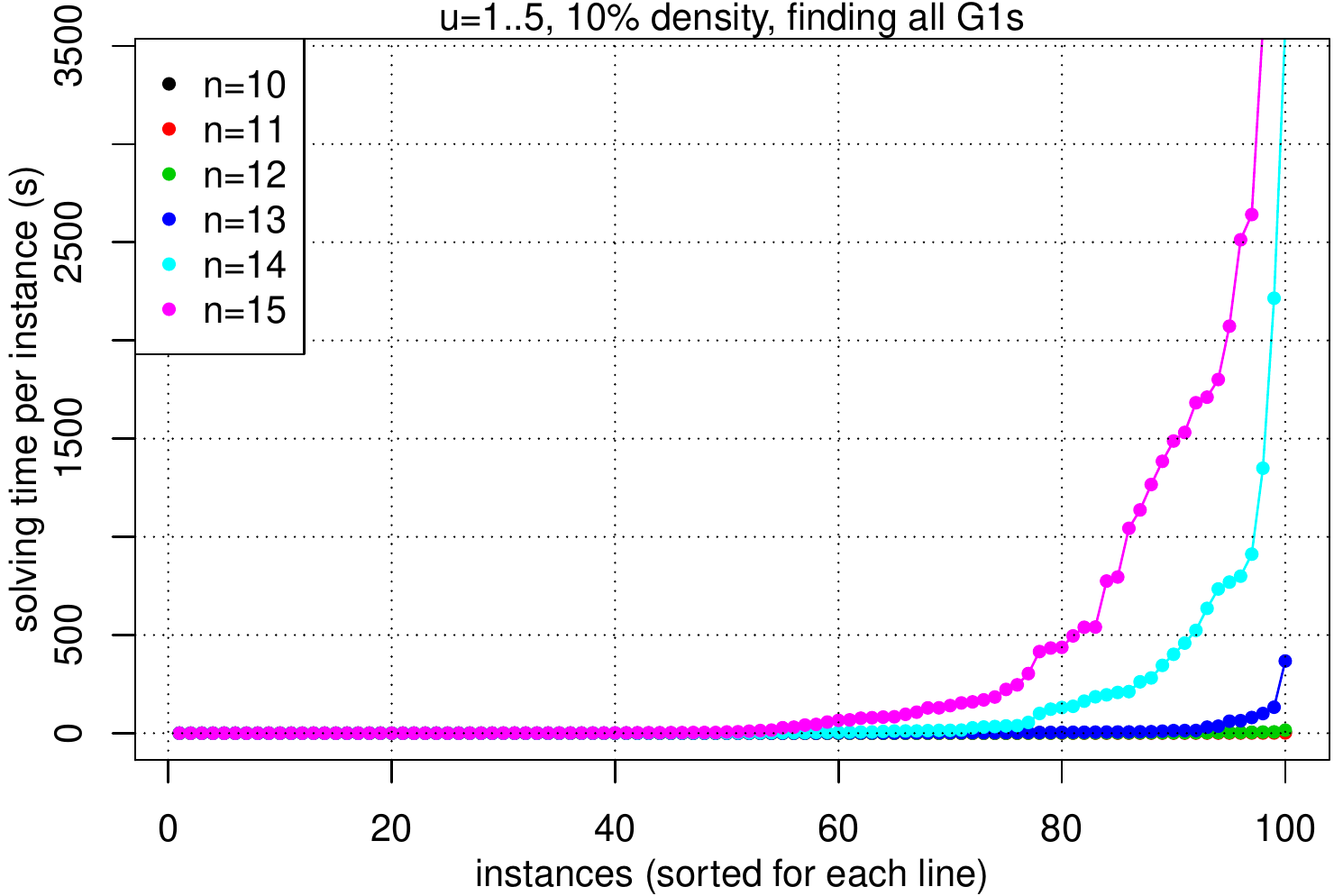}
\end{center}
\vspace{-.4cm}
\caption{Left: Influence of input graph density on running times of our approach. Right: Scalability of our approach when enumerating all solutions over 
$u=1,\ldots,5$.}
\label{fig:densities}\label{fig:u}
\vspace{-.4cm}
\end{figure}       
                                                                                      
The running times of the                                                                                     
MSL algorithm and our approach (SAT) on 10-node input graphs with different edge densities
are shown in Figure~\ref{fig:scalability-density}.
Figure~\ref{fig:scalability-nodes} (right) shows
the scalability of the two approaches in terms of increasing number of nodes in the input graphs and fixed 10\% edge density.
Our declarative approach clearly outperforms MSL.
10-node input graphs, regardless of edge density, are essentially trivial for our approach, while 
the performance 
of MSL deteriorates noticeably as the density increases.
For varying numbers of nodes in 10\% density input graphs, 
our approach scales
up to 65 nodes with a one hour time limit; even for 70 nodes, 25 graphs finished in one hour. In contrast, MSL reaches only 35 nodes; our approach  uses only a few seconds for those graphs.
The scalability of our algorithm  allows for investigating the influence of edge
density for larger graphs. 
Figure~\ref{fig:densities} (left) plots the running times of our approach (when enumerating \emph{all} solutions)
for $u=2$ on  20-node input graphs of varying densities. Finally, Figure~\ref{fig:u} (right) shows the scalability  of our approach in the more challenging task
of enumerating \emph{all} solutions over the \emph{range} $u=1,\ldots,5$ 
simultaneously.
  This also demonstrates
the generality of our approach:
it is not restricted to solving 
for individual values of $u$ separately.

\section{Learning from Undersampled Data}\label{sec:data2g1}

Due to statistical errors in estimating
$\cal{H}$ and the sparse distribution of $\cg{u}$ in ``graph space'', there will often be \emph{no} $\cg{1}$s that are consistent with $\cal{H}$. 
Given such an $\cal{H}$,
neither the MSL algorithm nor our approach in the previous section can output a solution, and they simply conclude that no solution $\cg{1}$ exists for the input
$\cal{H}$.
In terms of our constraint declarations, this is witnessed by conflicts among the constraints
for any possible solution candidate.
Given the inevitability of statistical errors, 
we should not simply conclude that no consistent $\cg{1}$ exists for such an $\cal{H}$.
Rather, we should aim to learn $\cg{1}$s that, in light of the underlying conflicts, are ``optimally close" (in some well-defined sense of ``optimality") to being consistent with $\cal{H}$. 
We now  turn to this more general problem setting, and propose what (to the best of our knowledge) is the 
first approach to learning, by employing constraint optimization, from undersampled data under conflicts.
In fact, we can use the ASP formulation already discussed---with minor modifications---to address this problem.

In this more general setting, the input consists of both the estimated graph $\cal{H}$, and also (i) weights $w(e\in \cal{H})$ indicating the reliability of edges present in $\cal{H}$; and (ii) weights $w(e\not \in \cal{H})$ indicating the reliability of edges absent in $\cal{H}$. 
Since $\cg{u}$ is $\cg{1}$ subsampled by $u$, the task is to find a $\cg{1}$ that minimizes the objective function:
\begin{eqnarray*}                                                                                                                                                                                       
f(\cg{1},u) &=&\sum_{e\in {\cal{H}}} I[e \not\in \cg{u}] \cdot w(e\in {\cal{H}}
)+
 \sum_{e \not\in {\cal{H}}} I[e \in \cg{u}] \cdot w(e\not\in {\cal{H}}),                                                                                                                                             
\end{eqnarray*}  
where the indicator function $I(c)=1$ if the condition $c$ holds, and $I(c)=0$ otherwise.
Thus, edges that differ between the estimated input $\cal{H}$ and the $\cg{u}$ corresponding to the solution $\cg{1}$
 are penalized by the weights representing the reliability of the measurement timescale estimates.                                                                                                            In the following, we first outline how the ASP encoding for the search problem without optimization
is easily generalized to enable finding optimal $\cg{1}$ with respect to this objective function.
We then describe 
alternatives for determining the 
weights $w$, and present simulation results on the relative performance of the different weighting schemes.

\subsection{Learning by Constraint Optimization}

To model the objective function for handling conflicts,
only simple modifications are needed to our ASP encoding:
instead of declaring \emph{hard} constraints that require that the paths induced by $\cg{1}$ \emph{exactly} correspond to the edges in $\cal{H}$,
we \emph{soften} these constraints by declaring that the violation of each individual constraint incurs
the associated weight as penalty. 
In the ASP language, this can be expressed by
augmenting the input predicates \texttt{\small edgeh(X,Y)}
with weights: \texttt{\small edgeh(X,Y,W)} (and similarly for \texttt{\small no\_edgeh}, \texttt{\small confh} and \texttt{\small no\_confh}).
Here the additional argument $W$ represents the weight $w((x \rightarrow y) \in {\cal{H}})$ given as input. The following expresses that each conflicting presence of an edge in $\mathcal{H}$ and $\cg{u}$ is penalized with the associated weight $W$.

\pagebreak

\vspace*{-10mm}

\begin{framed}
\footnotesize
\begin{verbatim}
 :~ edgeh(X,Y,W), not edgeu(X,Y). [W,X,Y,1]
 :~ no_edgeh(X,Y,W), edgeu(X,Y). [W,X,Y,1]
 :~ confh(X,Y,W), not confu(X,Y). [W,X,Y,2]
 :~ no_confh(X,Y,W), confu(X,Y). [W,X,Y,2]
\end{verbatim}  
\end{framed}             

\noindent
This modification provides an ASP encoding for Task~2; that is, the
optimal solutions to this ASP encoding correspond exactly to the $\cg{1}$s that minimize the objective function
$f(\cg{1},u)$ for any $u$ and input $\cal{H}$ with weighted edges.

\subsection{Weighting Schemes}

We use two different schemes for weighting the presences and absences of edges in $\cal{H}$ according to their reliability. To determine the presence/absence of an edge $X \rightarrow Y$ in $\cal{H}$ we simply test the corresponding independence $X^{t-1} \independent Y^t \given \mathbf{V}^{t-1}\setminus X^{t-1}$.
To determine the presence/absence of an edge $X \leftrightarrow Y$ in $\cal{H}$, we run the independence test: $X^t\independent Y^t \given \mathbf{V}^{t-1}$.

The simplest approach is to use uniform weights on the estimation result of $\cal{H}$:
\begin{eqnarray*}
w(e \in \cal{H}) &=& 1  \quad \forall e \in \cal{H},\\
w(e \not\in \cal{H}) &=& 1 \quad \forall e \not\in \cal{H}.
\end{eqnarray*}
Uniform edge weights resemble the search on the Hamming cube of $\cal{H}$ that \cite{PlisDY:UAI2015} used to address the problem of finding $\cg{1}$s when $\cal{H}$ did not correspond to any $\cg{u}$.

A more intricate approach is to use pseudo-Boolean weights following   \cite{hej2014,chain,Margaritis}.
They used Bayesian model selection to obtain reliability weights for independence tests. Instead of a $p$-value and a binary decision, these types of tests give a measurement of reliability for an independence/dependence statement as a Bayesian probability. We can directly use their approach of attaching log-probabilities as the reliability weights for the edges. For details, see Section 4.3 of \cite{hej2014}. Again, we only compute weights for the independence tests mentioned above in the estimation of $\cal{H}$.

\begin{figure}[t]
\vspace{-7mm}
\begin{center}
 \includegraphics[bb=0 0 714 430,width=1.0\textwidth]{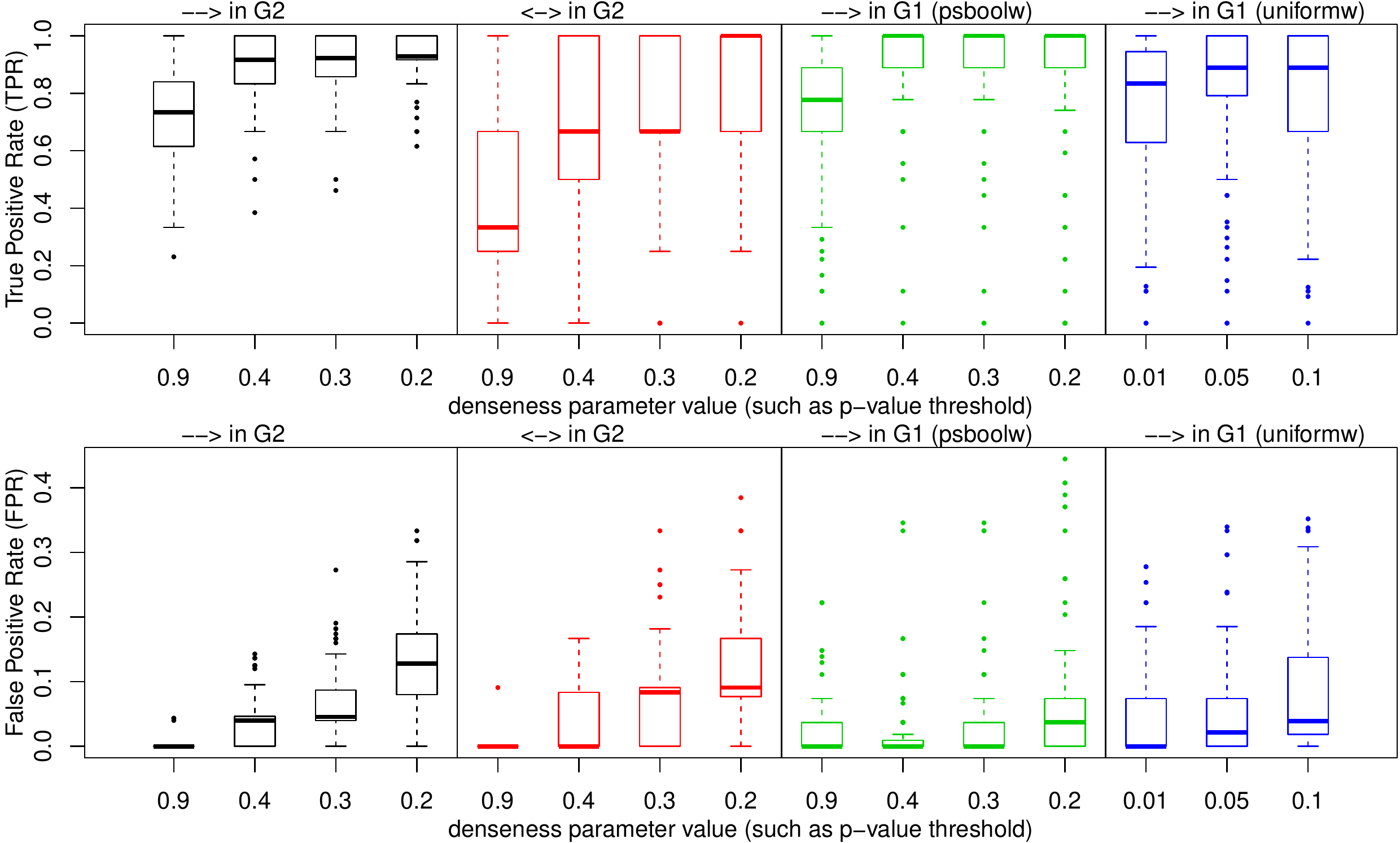}
\vspace{-1cm}
\end{center}
\caption{Accuracy of the optimal solutions with different weighting schemes and parameters (on x-axis). See text for further details. }\label{fig:accuracy}
\vspace{-.4cm}
\end{figure}

\subsection{Simulations}
\label{sec:simulations}

We use simulations to explore the impact of the choice of weighting schemes on the accuracy and runtime efficiency of our approach.
\label{sec:modelgeneration}
For the simulations, system timescale structures $\cg{1}$ and the associated data generating models were constructed in the following way. To guarantee connectedness of the graphs, we first formed a cycle of all nodes in a random order (following \cite{PlisDY:UAI2015}). We then randomly sampled additional directed edges until the required density was obtained. Recall that there are no bidirected edges in $\cg{1}$. We used Equations~\ref{eq1} and~\ref{eq2} to generate the measurement timescale structure $\cg{u}$ for a given $u$. When sample data were required, we used linear Gaussian structural autoregressive processes (order 1) with structure $\cg{1}$ to generate data at the system timescale, where coefficients were sampled from the two intervals $\pm [0.2, 0.8]$. We then discarded intermediate samples to get the particular subsampling rate.\footnote{
\texttt{Clingo} only accepts integer weights; we  multiplied weights by 1000 and rounded to the nearest integer.}

Figure~\ref{fig:accuracy} shows the accuracy of the different methods in one setting: subsampling rate $u=2$, network size $n=6$, average degree 3, sample size $N=200$, and 100 data sets in total. The positive predictions correspond to presences of edges; when the method returned several members in the equivalence class, we used mean solution accuracy to measure the output accuracy.
The x-axis numbers correspond to the adjustment parameters for the statistical independence tests ($p$-value threshold for uniform weights, prior probability of independence for all others). The two left columns (black and red) show the true positive rate and false positive rate of $\cal{H}$ estimation (compared to the true $\cg{2}$), for the different types of edges, using different statistical tests. 
For estimation from 200 samples, we see that the structure of $\cg{2}$ can be estimated with good tradeoff of TPR and FPR with the middle parameter values, but not perfectly. The presence of directed edges can be estimated more accurately. 
More importantly, the two rightmost columns in Figure~\ref{fig:accuracy} (green and blue) show the accuracy of $\cg{1}$ estimation. Both weighting schemes produce good accuracy for the middle parameter values, although there are some outliers. The pseudo-Boolean weighting scheme still outperforms the uniform weighting scheme, as it produces high TPR with low FPR for a range of threshold parameter values (especially for $0.4$).

Finally, the running times of our approach are 
shown in Figure~\ref{fig:weighting-scalability} with different weighting schemes, network sizes ($n$), and sample sizes ($N$). The subsampling rate was again fixed to $u=2$, and average node degree was 3. The independence test threshold used here corresponds to the accuracy-optimal parameters in Figure~\ref{fig:accuracy}. 
The pseudo-Boolean weighting scheme allows for much faster solving: for $n=7$, it finishes all runs in a few seconds (black line), while the uniform weighting scheme (red line) takes tens of minutes. Thus, the pseudo-Boolean weighting scheme provides the best performance in terms of both computational efficiency and accuracy.
Second, the sample size has a significant effect on the running times: larger sample sizes take \emph{less} time. For $n=9$ runs, $N=200$ samples (blue line) take longer than $N=500$ (cyan line). Intuitively, statistical tests should be more accurate with larger sample sizes, resulting in fewer conflicting constraints. For $N=1000$, the  global optimum is found here for up to 12-node graphs, though in a considerable amount of time.

\section{Conclusion}

\begin{figure}[!t]
\centering
 \vspace{-7mm}
 \includegraphics[width=0.58\columnwidth, bb=0 0 426 274]{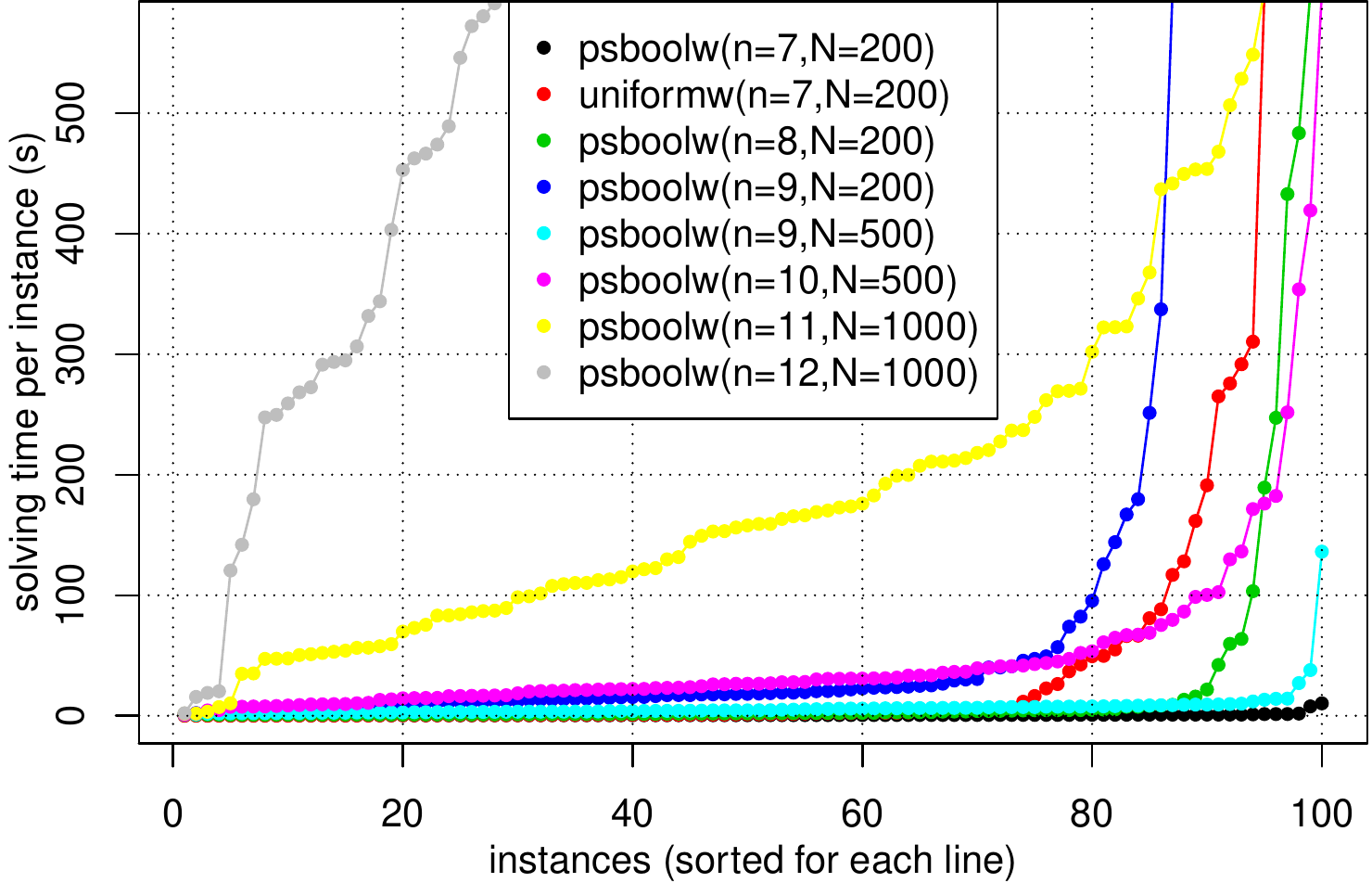}
 \vspace{-3mm}
\caption{Scalability of our 
approach under different settings.}\label{fig:weighting-scalability}
\vspace{-4mm}
\end{figure}

In this paper, we introduced a constraint optimization based solution for the problem of learning causal timescale structures from subsampled measurement timescale graphs and data. Our approach considerably improves the state-of-art
; in the simplest case (subsampling rate $u=2$), we extended the scalability by several orders of magnitude. 
Moreover, our method generalizes to handle different or unknown subsampling rates in a computationally efficient manner. Unlike previous methods, our method can operate directly on finite sample input, and we presented approaches that recover, in an optimal way, from conflicts arising from statistical errors.
We expect that this considerably simpler approach will allow for the relaxation of additional model space assumptions in the future. In particular, we plan to use this framework to learn the system timescale causal structure from subsampled data when latent time series confound our observations.

\acks{AH was supported by Academy of Finland Centre of Excellence in Computational Inference Research COIN (grant 251170).
SP was supported by NSF IIS-1318759 \& NIH R01EB005846.
MJ was supported by Academy of Finland Centre of Excellence in Computational Inference Research COIN (grant 251170) and 
grants 276412, 284591; and Research Funds of the University of Helsinki. FE was supported by NSF 1564330.
DD was supported by NSF IIS-1318815 \&  NIH  U54HG008540  (from  the  National  Human  Genome Research Institute through funds provided by the trans-NIH Big Data to Knowledge (BD2K) initiative). The content is solely the responsibility of the authors  and does not necessarily represent the official views of the National Institutes of Health.}

\bibliography{sat4u}

\end{document}